\documentclass[a4paper]{article}

\usepackage[margin=1in]{geometry}
\usepackage{graphicx, color}
\usepackage{amssymb}
\usepackage{mathtools}  
\usepackage{dsfont}
\usepackage{amsmath}
\usepackage{mismath}
\usepackage{pgfplots}
\usepackage{cite}
\pgfplotsset{compat=1.7}
\usepackage[numbers]{natbib}
\usepackage[utf8]{inputenc}
\usepackage{hyperref}

\usepackage[font=footnotesize]{caption}

\newtheorem{theorem}{Theorem}
\newtheorem{lemma}[theorem]{Lemma}
\newtheorem{definition}[theorem]{Definition}

\newcommand{\rcga}{$r$\nobreakdash-cGA\xspace}
\newcommand{\lo}{\textsc{LeadingOnes}\xspace}
\newcommand{\cga}{cGA\xspace}
\newcommand{\ronemax}{$r$\nobreakdash-OneMax\xspace}
\newcommand{\onemax}{OneMax\xspace}
\newcommand{\gonemax}{\textit{G}\nobreakdash-OneMax\xspace}
\newcommand{\ie}{i.\,e.\xspace}

\newcommand{\whp}{w.\,h.\,p.\xspace}
\newcommand{\expect}[1]{\textup{E}[#1]}
\newcommand{\prob}[1]{\textup{P}[#1]}
\newcommand{\probbig}[1]{\textup{P}\left[#1\right]}

\newcommand{\indic}[1]{\mathds{1}\{#1\}}
\newcommand{\filt}{\mathcal{F}_t}

\newenvironment{proof}{\begin{trivlist} {\item \textbf{Proof.}}}{\hspace*{\fill}$\qed$\end{trivlist}}
\newenvironment{proofwoqed}{\begin{trivlist}\item  {\textbf{Proof.}}}{\end{trivlist}}
\newenvironment{proofof}[1]{\begin{trivlist}\item \textit{Proof of #1:}}{\hspace*{\fill}$\qed$\end{trivlist}}

\usepackage[ruled,vlined,linesnumbered]{algorithm2e}
\SetKwComment{Comment}{/* }{ */}
\SetKwInOut{Input}{Initialization}
\SetKwInOut{Output}{Output}

\newcommand{\qed}{\quad\Box}

\title{Improved Runtime Analysis of a Multi-Valued Compact Genetic Algorithm on Two Generalized \text{OneMax} Problems}

\author{Sumit Adak \and
Carsten Witt}
\date{
	DTU Compute \\ Technical University of Denmark \\ Kgs. Lyngby, Denmark \\ \texttt{\{suad, cawi\}@dtu.dk}\\%
}

\begin{document}

\maketitle

\begin{abstract}
Recent research in the runtime analysis of estimation of distribution algorithms (EDAs) has focused on univariate EDAs for multi-valued decision variables. In particular, the runtime of the multi-valued \cga (\rcga) and UMDA on multi-valued functions has been a significant area of study. Adak and Witt (PPSN~2024) and Hamano et al.\ (ECJ~2024) independently performed a first runtime analysis of the \rcga on the $r$-valued \onemax function (\ronemax). Adak and Witt also introduced  a different $r$-valued \onemax function called \gonemax. 
However, for that function, only  empirical results were provided so
far due to the increased complexity of its runtime analysis, since \ronemax involves categorical values of two types only, while \gonemax encompasses all possible values.

In this paper, we present the first theoretical runtime analysis of the \rcga on the \gonemax function. We demonstrate that the runtime is $\bigo(nr^3\log^2 n \log r)$ with high probability. Additionally, we refine the previously established runtime analysis of the \rcga on \ronemax, improving the previous bound 
to $\bigo(nr \log n\log r)$, 
which improves the state of the art by an asymptotic factor of~$\log n$ and  
is tight for the binary case. Moreover, we for the first time include  
the case of frequency borders.

\end{abstract}

\section{Introduction}
\label{section:intro}

Estimation of distribution algorithms (EDAs) have gained significant attention in recent years as a powerful class of optimization techniques. In particular, univariate EDAs~\cite{baluja1994population,muhlenbein1996recombination,harik1999compact,Krejca2020}, which focus on multi-valued decision variables \cite{AdakPPSN2024,BENJEDIDIA2024114622}, have been a key area of study for improving algorithmic efficiency in solving complex problems. These algorithms have demonstrated promise in various domains such as combinatorial optimization, and others. Among these, the multi-valued compact genetic algorithm (\rcga)~\cite{AdakPPSN2024} and the multi-valued univariate marginal distribution algorithm ($r$-UMDA)~\cite{BENJEDIDIA2024114622} have been explored for their runtime performance on multi-valued functions, providing valuable insights into their capabilities and limitations in different problem settings.

Despite the growing interest in EDAs, the runtime analysis of these algorithms, especially when applied to functions with a large range of possible values, remains an area in need of deeper exploration. One significant contribution in this area is the work by Ben~Jedidia et al.\ \cite{BENJEDIDIA2024114622}, who conducted the first theoretical runtime analysis of the $r$-UMDA on the $r$-valued \lo function. Following this, Adak and Witt~\cite{AdakPPSN2024} and, independently, Hamano et~al.~\cite{HamanoECJ24} provided the first theoretical runtime analysis of the \rcga on the $r$-valued \onemax function (\ronemax), interpreted as categorical variables 
with only two settings. 
Building upon this, Adak and Witt also introduced the $r$-valued generalized \onemax function (\gonemax), which extends the \ronemax function by considering all possible values for the decision variables. This extension posed a more complex scenario for runtime analysis due to the larger search spaces and increased diversity of potential solutions. Therefore, no 
runtime analysis on \gonemax was given.

In this work, we address this gap by presenting the first theoretical runtime analysis of the \rcga on the \gonemax function. We show that the runtime of the \rcga on \gonemax is $\bigo(nr^3\log^2 n \log r)$ with high probability (\whp). Additionally, we improve upon the existing runtime analysis of the \rcga on \ronemax, refining the previously 
best bound by Hamano et al.~\cite{HamanoECJ24} from $\bigo(nr \log^2 (nr))$ 
to $\bigo(nr\log n \log r)$, effectively eliminating a $\log n$ factor from the original estimate and 
leading to the tight bound $\bigo(n\log n)$ in the binary case \cite{sudholt2019choice}. Our work not only extends the understanding of the \rcga's performance on more complex problems but also opens the door for future research on runtime optimization for EDAs in multi-valued settings. Moreover, the analysis on \ronemax for 
the first time 
includes the so-called borders on frequencies that are commonly used in the theoretical literature to guarantee 
finite expected optimization time; although for a parameter choice that does not lead to tight bounds. Finally, our revised analysis for the first time avoids the so-called 
well-behaved frequencies assumption, which is generally ruled out in the \rcga due to its more complex 
update of the multi-valued frequency model.

The manuscript is structured as follows. Section~\ref{section:preliminaries} defines the multi-valued \onemax functions and and the multi-valued cGA.
Section~\ref{section:framework} describes probabilistic 
properties of the algorithm.
In Section~\ref{section:runtime}, we present the primary technical results, including genetic drift analysis and the runtime analyses of the \rcga on the \ronemax and \gonemax functions. Section~\ref{section:experiment} explores the empirical runtime for different values of the parameter $K$, which represents the hypothetical population size of the \rcga. Finally, the manuscript concludes with a summary. 

\section{Preliminaries}
\label{section:preliminaries}

We analyze the $r$-valued compact genetic algorithm (\rcga) with the goal of maximizing $r$-valued functions $f\colon \{0,1,\dots,r-1\}^{n}\rightarrow \mathbb{R}$ where $f(x)$ denotes the \emph{fitness} of an individual $x \in \{0,1,\dots,r-1\}^{n}$. In this setting, the vector $x=(x_1,x_2,\dots,x_n)$ represents a candidate solution consisting of $n$ values, each from the set $\{0,1,\dots,r-1\}$. 

This work extends the analysis of the \onemax function by considering multi-valued variants, as independently considered by Adak and Witt \cite{AdakPPSN2024} and Hamano et al.~\cite{HamanoECJ24}. Following the notation from 
the first paper, these generalizations are called the \ronemax and \gonemax functions and 
accommodate the larger value set $\{0,1,\dots,r-1\}$, where $r\in \mathbb{N}_{\geq 2}$ and $n\in \mathbb{N}_{\geq 1}$. Specifically, we define \ronemax and \gonemax as follows:
\begin{align*}
r\text{-OneMax}(x) \coloneqq  \sum_{i=1}^{n} \indic {x_{i} = r-1}\
 \text{\ and\ \ }
G\text{-OneMax}(x) \coloneqq  \sum_{i=1}^{n} x_i.
\end{align*}

The single maximum in both the \ronemax and \gonemax functions is the all-$(r-1)$s string. The primary distinction between these two functions lies in how they handle the contribution of each position in the string to the overall fitness. In \ronemax, the values are categorical, which means that only the correct value for a position (especially, $r-1$) contributes to the fitness. In other words, \ronemax evaluates each position independently and assigns a fitness signal only if the value at that position matches the optimal value, $r-1$. This leads to a fitness value of $n$ for the optimal string.

In contrast, \gonemax employs a finer distance metric between values, where each position contributes a fitness signal toward the optimal value. Specifically, for each position, the contribution to fitness is proportional to how close the value is to the optimal value, $r-1$. This means that \gonemax considers all values in the set $\{0,1,\dots,r-1\}$ for each position, with fitness increasing as the value approaches $r-1$. As a result, the maximum fitness for the \gonemax function is $n(r-1)$. 
The function \gonemax was only defined in 
\cite{AdakPPSN2024} but without a concrete runtime analysis.

The compact genetic algorithm (cGA), introduced by Harik et al.\ \cite{harik1999compact}, is a popular univariate estimation of distribution algorithm that operates with a compact representation of the population. Unlike traditional evolutionary algorithms, which maintain a population of candidate solutions, \cga approximates the distribution of solutions using a probabilistic model. 
It relies on a single parameter $K\in\mathbb{R}_{>0}$, often referred to as the hypothetical population size \cite{doerr2021runtime}. This parameter describes 
the strength of updates to the probabilistic model in 
an iteration of the algorithm.

An extension of the \cga is the \rcga, introduced 
in \cite{BENJEDIDIA2024114622} and independently 
analyzed  in \cite{HamanoECJ24} and \cite{AdakPPSN2024}, which supports multi-valued variables instead of only binary ones. 
See Algorithm~\ref{algorithm:r-cGA-rOneMax}.

\begin{algorithm}
\caption{$r$-valued Compact Genetic Algorithm ($r$-cGA) for the maximization of $f : \{0,\dots,r-1\}^n \rightarrow \mathbb{R}$}
\label{algorithm:r-cGA-rOneMax}
\Input{$t \gets 0$; $p^{(t)}_{i,j} \gets \frac{1}{r}$, where $(i,j)\in \{1,\dots,n\}\times \{0,\dots,r-1\}$}
\While{termination criterion not met}{
\For{$i\in \{1, 2, \dots, n\}$ independently}{
$x_{i} \gets j$ with probability $p^{(t)}_{i,j}$ w.r.t. $j=0,\dots,r-1$ \\
$y_{i} \gets j$ with probability $p^{(t)}_{i,j}$ w.r.t. $j=0,\dots,r-1$
}
\If{$f(x) < f(y)$}{swap $x$ and $y$}
\For{$i\in \{1, 2, \dots, n\}$}{
\For{$j\in \{0, 1, \dots, r-1\}$}{
$p^{(t+1)}_{i,j} \gets p^{(t)}_{i,j} + \frac{1}{K} (\indic{x_{i} = j} - \indic{y_{i}=j})$}
\If{\textup{borders are used}}{$p^{(t+1)}_{i,j} \gets$ restrict $p^{(t+1)}_{i,j}$ to 
$[\frac{1}{(r-1)n}, 1-\frac{1}{n}]$ (\cite{BENJEDIDIA2024114622})}
}
$t\gets t + 1$
}
\end{algorithm}

The probabilistic model of the \rcga is defined by an $n\times r$ \textit{frequency matrix}, where each row $i\in \{1,\dots,n\}$ corresponds to a frequency vector $p_i\coloneqq (p^{(t)}_{i,j})_{j\in \{0,\dots,r-1\}}$ representing the marginal probabilities of each value $j\in \{0,\dots,r-1\}$ at position $i$ at time $t$. Initially, each element of the  matrix is set to $1/r$, leading to a uniform probability distribution for the first individuals  sampled. After generating the individuals, their fitness values are compared. Depending on which individual has a higher fitness, the frequencies are updated by $\pm K$. 
Specifically, the frequency vector of the better individual (that one with higher fitness) has the corresponding frequency at each position increased by $1/K$, while the corresponding frequency for the worse individual is decreased by the same amount.

In this work, we for the first time analyze the \rcga on \ronemax in the presence of borders on the marginal frequencies, which are commonly used in the theoretical literature \cite{sudholt2019choice,doerr2021runtime}. They  
restrict the frequencies to the interval $[1/((r-1)n),1/n]$ 
to avoid frequencies being irreversibly stuck 
at~$0$ or~$1$. On $r$-valued spaces, this restriction 
(also called capping) has to be done in a careful way. 
We follow the original procedure from \cite{BENJEDIDIA2024114622}
to cap the frequencies. 
Note that this
may also update frequencies belonging to values in~$\{0,\dots,r-1\}$ that 
were not sampled in any of the two individuals. The restriction 
ensures that there is always a positive probability to sample 
an individual of optimum value; on the negative side, the more 
complicated update mechanism for $r\ge 3$ rules out the 
so-called well-behaved frequency assumption (\cite{Doerr-FOGA-19}, discussed in Section~\ref{subsection:runtimecgagonemax})
 that has been useful for the binary 
\cga. In the rest of the paper, we denote $1/((r-1)n)$ as lower 
border and $1-1/n$ as upper border.

We are interested in the runtime (also called optimization time) of the algorithm, which refers to the number of function evaluations required to sample a solution with optimal fitness. This runtime is proportional to the value of $t$ in the algorithm, where $t$ denotes the number of iterations or updates the algorithm performs.

\section{Framework for the Analysis}
\label{section:framework}

In optimization, understanding how algorithms navigate solution spaces is crucial. This section explores the \rcga and its use of probability vector updates. We focus on two types of updates, random-walk steps and biased steps, as already considered in \cite{sudholt2019choice,AdakPPSN2024}, and their role in guiding the algorithm toward optimal solutions. The following analysis details how these updates influence the optimization of the \rcga on the \gonemax function. 

The change in $p^{(t)}_{i,j}$ at each step is defined as $\Delta_{i,j}\coloneqq p^{(t+1)}_{i,j} - p^{(t)}_{i,j}$. This allows us to express the change at position~$i$ for value $r-1$ as $\Delta_{i,r-1}\coloneqq p^{(t+1)}_{i,r-1} - p^{(t)}_{i,r-1}$. The update is influenced by whether the value at position~$i$ affects the decision to update, based on the two strings $x$ and $y$ sampled at time~$t$. Now, we examine the changes in values at all positions except position~$i$. To capture this, we define $D_{i} := \left( \sum_{j\neq i} x_{j} - \sum_{j\neq i} y_{j} \right)$ for the function \gonemax. At this stage, the \rcga can undergo two distinct kinds of updates. One is \emph{random-walk step (rw-step)} and another is \emph{biased step (b-step)}. 

A random-walk step in the \rcga refers to an update where the probability vector experiences small, stochastic fluctuations. These fluctuations occur due to the random nature of the sampled strings $x$ and $y$, without a strong directional bias toward improving the objective function. In contrast, a biased step refers to an update in which the changes in the probability vector are influenced by directional tendencies, aligning with the optimization objective. Particularly, the update of the  probability vector is directional and aligned with the fitness of the sampled solutions. If one solution is better (has higher fitness), the probability vector is adjusted to favor that solution, which drives the algorithm closer to an optimal solution. Together, these steps characterize the behavior of the \rcga during its optimization process. The analysis that follows, along with the terminology random-walk and biased step, is closely based on the methodology outlined in \cite{AdakPPSN2024} for the \ronemax function. \\
\textbf{Random-walk steps:}
If $D_{i} \geq r-1$, then position~$i$ does not influence the decision to update with respect to strings $x$ and $y$. When $\Delta_{i,j}\neq 0$, it is required that the value at position~$i$ for $j$ is sampled differently. This implies that the value of $p^{(t)}_{i,j}$ will either increase or decrease by $1/K$ (unless a border is hit) with equal probability $p^{(t)}_{i,j} (1 - p^{(t)}_{i,j})$. Otherwise, it holds the same frequency value. This can be represented by introducing a variable $F_i$, where 
\[F_{i}\coloneqq \begin{cases}
+1/K & \text{with probability } p^{(t)}_{i,j} (1 - p^{(t)}_{i,j}),\\
-1/K & \text{with probability }  p^{(t)}_{i,j} (1 - p^{(t)}_{i,j}),\\
0 & \text{with the remaining probability.}
\end{cases}\]

The above step is called a random-walk step. On the 
one hand, it is characterized by  $\E(\Delta_{i,j}\mid p^{(t)}_{i,j}, D_{i} \geq r-1) = \E(F_{i}\mid p^{(t)}_{i,j}) = 0$. On the other hand, when the value of $D_{i} < -(r-1)$, then the strings are swapped; however, the outcome at position~$i$ doesn't affect the ranking anymore. Consequently, this step becomes a random-walk step, too. Therefore, as before, the same reasoning remains applicable.\\
\textbf{Biased steps:} If $D_{i} = 0$, then the strings $x$ and $y$ are swapped unless the value at position~$i$ is sampled as $x_i\ge y_i$. In this case, both events of sampling position~$i$ differently increase the $p^{(t)}_{i,j}$ value. As a result, $\Delta_{i,j} = 1/K$ occurs with a probability of $2p^{(t)}_{i,j} (1 - p^{(t)}_{i,j})$. Let us define a random variable $B_i$ such that:
\[B_{i}\coloneqq \begin{cases}
+1/K & \text{with probability } 2p^{(t)}_{i,j} (1 - p^{(t)}_{i,j}),\\
0 & \text{with the remaining probability.}
\end{cases}\]

Therefore, a biased step occurs under the above condition and is represented by the following expression: $\E(\Delta_{i,j}\mid p^{(t)}_{i,j}, D_{i}=0) = \E(B_{i}\mid p^{(t)}_{i,j}) = 2p^{(t)}_{i,j} (1 - p^{(t)}_{i,j})/K$. Furthermore, when $D_i\in \{-(r-1), \dots, -1,1,\dots, r-2\}$, the outcome is determined by the values of $x_i$ and $y_i$. Specifically, there are two cases: if $(y_i-x_i) > D_i$, a biased step is performed; otherwise, the process proceeds with a random-walk step. 

The framework for analyzing the \rcga on \ronemax has already been established by Adak and Witt~\cite{AdakPPSN2024}, and we 
follow this here in the revised analysis 
 the \rcga on the \ronemax problem.

\section{Analysis of Runtime}
\label{section:runtime}

In this section, we explain how to manage genetic drift using the negative drift theorem. We also present the runtime analysis of the \rcga on \gonemax, along with an improved runtime analysis of the \rcga on \ronemax without borders. Furthermore, we extend this to the case with borders, at the 
expense of a weaker bound.

\subsection{Controlling Genetic Drift through Negative Drift}
\label{subsection:drift} 

Genetic drift refers to random fluctuations in the frequency of genetic traits, which can lead to optimal solutions in evolutionary algorithms. Controlling this drift is essential for ensuring that the algorithm converges to the desired solution without being disrupted by randomness. Genetic drift in EDAs has been thoroughly explored in various runtime analyses \cite{doerr2020univariate,droste2005not,lengler2021complex,sudholt2019choice,witt2018domino,witt2019upper}. In previous research \cite{AdakPPSN2024}, \textit{martingale} techniques were used to control genetic drift effectively. However, in this case, the martingale techniques provide weaker results compared to methods that take advantage of the stochastic drift of a frequency toward higher values. In this work, we employ a \textit{negative drift theorem} to bound genetic drift. Note
that negative drift was also the preferred 
tool in \cite{HamanoECJ24}.

The negative drift theorem is motivated by the need to analyze the behavior of a stochastic process $(X_t)_{t\ge 0}$ that evolves over time $t$ within a specific interval $[a,b]\subseteq \mathbb{R}$.
Our goal is to demonstrate, using negative drift, that while 
a frequency value passes through this interval, it will not drop significantly below its initial value at $a$. This will ensure that the frequency remains relatively stable within the bounds of the interval, even as it undergoes fluctuations.

The following negative drift theorem, which does not include self-loops, applies to processes with a bounded step size. It has been taken from~\cite[Theorem 4]{HamanoECJ24} and has been trivially modified to reverse the direction of the process and to permit arbitrary drift intervals $[a,b]$ instead of the fixed limit of~0 on one side of the interval. 

\begin{theorem}
\label{theorem:negativedrift}
Consider a stochastic process \( (X_t, \mathcal{F}_t)_{t \in \mathbb{N}_0} \) with \( X_0 \geq b \). Let $a\le b$. Let \( T \) be a stopping time defined as
\[
T = \min \{ t \in \mathbb{N} : X_t \leq a \}.
\]

Assume that, for all \( t \), there are constants \( 0 < \varepsilon < (b-a)/2 \) and \( 0 < c < b-a \) satisfying
\[
\mathbb{E}[X_{t+1} - X_{t} \mid \mathcal{F}_{t}] \geq \varepsilon \Pr(X_{t+1} \neq X_{t} \mid \mathcal{F}_{t}),
\]
\vspace{-1.7\bigskipamount}
\[
|X_{t+1} - X_{t}| \leq c \quad \text{w.\,p. 1.}
\]
Then, for all \( t \geq 0 \),
\[
\Pr(T \leq t) \leq \frac{2(b-a)t}{\varepsilon} \exp\left(-\frac{(b-a)\varepsilon}{4c^2}\right).
\]
\end{theorem}

Now, we apply the negative drift theorem to bound the probability that the frequency falls below $a=(k/r)-1/(2r)$ when started at $b=k/r$, in the context of the \rcga applied to the \gonemax problem. In this case, we obtain the following lemma.

\begin{lemma}
\label{lemma:negativedrift-gonemax} 
Consider the \rcga applied to the \gonemax problem, and assume that the frequency $p^{(t)}_{i,r-1}$ has reached a value of at least $k/r$ ($k>0$). Then, if $K\ge cr^2\sqrt{n}\log n$ is chosen where $c$ is a sufficiently large constant and $K,r=\mathit{poly}(n)$, the probability that $p^{(t)}_{i,r-1}$ drops below $(k/r)-1/(2r)$ within the next $T^{*}=c'Kr\sqrt{n}(\log r+\log n)$ steps is at most $n^{-c''}$ where $c''\ge K/(288r^2\sqrt{n}\ln n)$.
\end{lemma}

\begin{proof}
In the negative drift theorem (Theorem~\ref{theorem:negativedrift}), we set $a=(k/r)-1/(2r)$ and $b=k/r$, so $b-a=1/(2r)$. Using Lemma~\ref{lemma:sdrift} (which is part of the next subsection) and noting that the probability of changing the frequency is $2p^{(t)}_{i,r-1}(1-p^{(t)}_{i,r-1})$, we find that 
$\varepsilon = \frac{4}{9K (2(r-1)\sqrt{3n}+1)}\ge \frac{4}{9K (2r\sqrt{4n})}\ge \frac{1}{9Kr\sqrt{n}}$, and $c=1/K$.

Therefore, the probability that $p^{(t)}_{i,r-1}$ never drops below $a$ in~$t$ steps is 
\begin{align*}
 \Pr(T \leq t) & \leq \frac{2(b-a)t}{\varepsilon} \exp\left(-\frac{(b-a)\varepsilon}{4c^2}\right)
  \leq t\cdot e^{\ln (9K\sqrt{n})-\frac{K}{72r^2\sqrt{n}}}.
\end{align*}

Further, since $K\sqrt{n} = poly(n)$, we have $\ln(K\sqrt{n}) \le c''' \ln n$ for some large constant $c'''$. On the other hand, we have $K/(72r^2\sqrt{n}) \ge c''\ln n$ for another constant $c''$ that can be arbitrarily large depending on our choice of $c$ in the assumption $K\ge cr^2 \sqrt{n}\log n$. Hence, for sufficiently large $c$ that  $K/(72r^2 \sqrt{n}) \ge 2 \ln(9K\sqrt{n})$. Therefore, the above exponential term is no larger than  $e^{-K/(144r^2\sqrt{n})}$. If $c'$ is large enough, $e^{K/(288r^2\sqrt{n})} \ge T^{*} = c'Kr\sqrt{n}(\log r + \log n)$. Then, we can choose $t\le e^{K/(288r^2\sqrt{n})}$. The total failure probability is 
\begin{align*}
 \Pr(T \le T^{*}) & \leq e^{\frac{K}{288r^2\sqrt{n}}}\cdot e^{-\frac{K}{144r^2\sqrt{n}}}
 \le e^{-\frac{K}{288r^2\sqrt{n}}}.
\end{align*}

As a result, the negative drift theorem provides a failure probability of $e^{-\frac{K}{288r^2\sqrt{n}}} \le n^{-c''}$ where $c''\ge K/(288r^2\sqrt{n}\ln n)$.
\end{proof}

Using martingale techniques like 
in \cite{AdakPPSN2024} instead of the negative drift method, we would arrive at (asymptotically) the same probability bound when applying these results in the context of the \gonemax problem. 
However, for the \ronemax problem the negative drift method provides a better bound than the martingale approach. We derive the following lemma for \ronemax using the negative drift method. To prove this lemma, we can apply the same technique used in the proof of Lemma~\ref{lemma:negativedrift-gonemax}.

\begin{lemma}
\label{lemma:negativedrift-ronemax} 
Consider the \rcga applied to the \ronemax problem, and assume that the frequency $p^{(t)}_{i,r-1}$ has reached a value of at least $k/r$ ($k>0$). Then, if $K\ge cr\sqrt{n}\log n$ is chosen where $c$ is a sufficiently large constant, and $K,r=\mathit{poly}(n)$, the probability that $p^{(t)}_{i,r-1}$ drops below $(k/r)-1/(2r)$ within the next $T^{*}=c'K\sqrt{n}\log r$ steps is at most $n^{-c''}$ where $c''\ge K/(288r\sqrt{n}\ln n)$.
\end{lemma}

\begin{proof}
This proof is similar to 
the proof of Lemma~\ref{lemma:negativedrift-gonemax}. 
In the context of the negative drift theorem (Theorem~\ref{theorem:negativedrift}), we set $a=(k/r)-1/(2r)$ and $b=k/r$, so the $b-a=1/(2r)$. Using Lemma~\ref{lemma:sdrift-ronemax}, we find that 
$\varepsilon = \frac{4}{9K (2\sqrt{3n}+1)}\ge \frac{4}{9K (2\sqrt{4n})}\ge \frac{1}{9K\sqrt{n}}$  and $c=1/K$.

Therefore, the probability that $p^{(t)}_{i,r-1}$ never drops below $a$ in~$t$ steps is 
\begin{align*}
 \Pr(T \leq t) & \leq \frac{2(b-a)t}{\varepsilon} \exp\left(-\frac{(b-a)\varepsilon}{4c^2}\right)\\
 & \leq t\cdot e^{\ln \left(\frac{9K\sqrt{n}}{r}\right)-\frac{K}{72r\sqrt{n}}}
\end{align*}

Further, since $K\sqrt{n} = poly(n)$, we have $\ln(K\sqrt{n}) \le c''' \ln n$ for some large constant $c'''$. On the other hand, we have $K/(72r\sqrt{n}) \ge c''\ln n$ for another constant $c''$ that becomes arbitrarily large depending on our choice of $c$ in the assumption $K\ge cr \sqrt{n}\log n$. Hence, we have for sufficiently large $c$ that  $K/(72r \sqrt{n}) \ge 2 \ln(9K\sqrt{n}/r)$. Therefore, the above exponential term is at most  $e^{-K/(144r\sqrt{n})}$. If $c'$ is large enough, $e^{K/(288r\sqrt{n})} \ge T^{*} = c'K\sqrt{n}\log r$.
Then, we can choose $t\le e^{K/(288r\sqrt{n})}$. Then, total failure probability is 
\begin{align*}
 \Pr(T \le T^{*}) & \leq e^{\frac{K}{288r\sqrt{n}}}\cdot e^{-\frac{K}{144r\sqrt{n}}}\\
  & \le e^{-\frac{K}{288r\sqrt{n}}}
\end{align*}

As a result, the negative drift theorem provides a failure probability of $e^{-\frac{K}{288r\sqrt{n}}} \le n^{-c''}$ where $c''\ge K/(288r\sqrt{n}\ln n)$.
\end{proof}

\subsection{Runtime of the \rcga on \gonemax}
\label{subsection:runtimecgagonemax}

To prove the main theorem, we first require the following lemmas. 

\begin{lemma}
\label{lemma:positiveupdate}
Let $p^{(t)}_{i,j}$ and $p^{(t+1)}_{i,j}$ represent the frequency vectors of the current and the next iteration of the $r$-cGA on the \gonemax function, where $(i,j)\in \{1,\dots,n\}\times \{0,\dots,r-1\}$. For sufficiently large $n$, the probability that $D_{i}=0$ is bounded below by:  
\[\P[D_{i}=0] \geq \frac{4}{9\left(2(r-1)\sqrt{3n}+1\right)}\]
where $D_{i}\coloneqq$ $\left( \sum_{j\neq i} x_{j} - \sum_{j\neq i} y_{j}\right)$.
\end{lemma}

\begin{proof}
Let  $x\in \{0,\dots,r-1\}^n $ and $y\in \{0,\dots,r-1\}^n$ represent two solutions from an iteration of the \rcga. Define $D_i$ as the difference between the number of entries in all the positions except position~$i$. Specifically, let: $X\coloneqq \sum_{j\neq i} x_j$ and $Y\coloneqq \sum_{j\neq i} y_j$, then the difference defined as $D_{i}\coloneqq X - Y$. We aim to estimate the probability that $D_{i}=0$.

Next, we calculate the variance, which is given by: $\Var (X) = \sum_{j\neq i} \Var (X_j)$. We can get $\E(X_i)=\sum_{j=0}^{r-1} p^{(t)}_{i,j}\cdot j$. Now, we compute the variance of $X_i$ as:
\begin{align*}
\Var(X_i) & = \E((X_i)^2)- (\E(X_i))^2 
 = \sum_{j=0}^{r-1} p^{(t)}_{i,j} \cdot j^2 - \left (\sum_{j=0}^{r-1}p^{(t)}_{i,j}\cdot j\right )^2 \\
& \leq \sum_{j=0}^{r-1} p^{(t)}_{i,j} \cdot j^2
\leq (r-1)^2 \sum_{j=0}^{r-1} p^{(t)}_{i,j} = (r-1)^2.
\end{align*} 

Further, we get
\begin{align*}
\Var(X) & = \sum_{j\neq i} \Var (X_j) \leq (n-1)(r-1)^2 
\end{align*} 

Further, by applying Chebyshev's inequality, we can estimate the probability that $X$ deviates from its mean by more than a specified multiple of its standard deviation $\sigma$. Let ${\sigma}^2 = \Var(X)$. Then, we obtain the following:
\[\P\left[ \left\vert X - \mu_{\setminus i} \right\vert \geq \sqrt{3}\sigma \right] \leq \frac{1}{3}\]
where $\mu_{\setminus i} = \sum_{j\neq i}\sum_{k=0}^{r-1} p^{(t)}_{j,k}\cdot k$.

Let $I := \left[ \mu_{\setminus i} - \sqrt{3}\sigma, \mu_{\setminus i} + \sqrt{3}\sigma  \right]$, and observe that there at most $2\sqrt{3}\sigma +1$ integers in the interval $I$. Assume that $X\in I$ and $Y\in I$, which occurs with a probability of at least $(1-1/3)^2 = 4/9$. Additionally, observe that $D_{i} = 0$ is equivalent to $X=Y$, and we compute: 
\begin{align*}
\P[X=Y \mid X,Y \in I] & = \sum_{z\in I} \P[X=z \mid X\in I] \cdot \P[Y=z\mid Y\in I]  \\
 & = \sum_{z\in I} \P[X=z \mid X\in I]^2 \\
 & \geq \sum_{z\in I} {\left( \frac{1}{\vert I \vert}\right)}^2 = \frac{1}{\vert I \vert} \geq \frac{1}{2\sqrt{3}\sigma +1}.
\end{align*} 
The first inequality holds due to the convexity of the square function, which means that shifting the probability mass to the average value $1/I$ minimizes the sum of squares. Therefore, the unconditional probability $\P[X=Y]=\P[D_{i}=0]$ is at least: ${4}/{(9(2\sqrt{3}\sigma +1))}$.
Substituting for $\sigma$, we obtain: 
\begin{align*}
\P[D_{i}=0] \geq \frac{4}{9\left(2\sqrt{3(n-1)(r-1)^2}+1\right)} \geq \frac{4}{9\left(2(r-1)\sqrt{3n}+1\right)}
\end{align*}
if $n$ is large enough.
\end{proof}

The proof of the previous lemma is similar to the analysis in \cite[Lemma~2]{AdakPPSN2024}; however, the sampling variance derived 
here is by an asymptotic factor of $r^2$ greater than 
for the case of \ronemax, which results in the factor $(r-1)$ 
in the denominator.

The next lemma is an important drift bound for frequencies.

\begin{lemma}
\label{lemma:sdrift}
 If $\frac{1}{K} \leq p^{(t)}_{i,r-1} \leq 1 -\frac{1}{K}$ and $\Delta_{i,r-1} = p^{(t+1)}_{i,r-1} - p^{(t)}_{i,r-1}$, then   
\[\E(\Delta_{i,r-1}\mid p^{(t)}_{i,r-1}) \ge \frac{8\left(p^{(t)}_{i,r-1} (1 - p^{(t)}_{i,r-1})\right)}{9K \left(2(r-1)\sqrt{3n}+1\right)}\]
 where $i\in\{1,\dots,n\}$.
\end{lemma}

\begin{proofwoqed}
From Section~\ref{section:framework}, we derive the following inequality, which depends on the position~$i$:
\[\Delta_{i,r-1} = F_{i}\cdot {R_i} + B_{i}\cdot {\overline{R_i}}.\]
where $R_i$ is an indicator for the event of a random-walk step.

Further, the expected change for a position with a specific value is
\[\E(\Delta_{i,r-1}\mid p^{(t)}_{i,r-1}) = \E(F_{i}\mid p^{(t)}_{i,r-1})\cdot \E({R_i}) + \E(B_{i}\mid p^{(t)}_{i,r-1})\cdot \E({\overline{R_i}}).\]
From the random-walk step, we know that $\E(F_{i}\mid p^{(t)}_{i,r-1})=0$ and from the biased step, we have $\E(B_{i}\mid p^{(t)}_{i,r-1}) = 2p^{(t)}_{i,r-1} (1 - p^{(t)}_{i,r-1})/K$. Additionally, from Lemma~\ref{lemma:positiveupdate}, we know that 
\[\E({\overline{R_i}})\geq \P[D_{i}=0] \geq \frac{4}{9\left(2(r-1)\sqrt{3n}+1\right)}.\]
Multiplying the results together, we derive the following expression:
\[\E(\Delta_{i,r-1}\mid p^{(t)}_{i,r-1}) \geq \frac{8\left(p^{(t)}_{i,r-1} (1 - p^{(t)}_{i,r-1})\right)}{9K \left(2(r-1)\sqrt{3n}+1\right)}.\quad\qed\]
\end{proofwoqed}

The previous drift bound from Lemma~\ref{lemma:sdrift} includes a term of 
$\Theta(1/(r\sqrt{n}))$, which stems 
from the upper bound $\bigo(nr^2)$ on 
the sampling variance, which is used to derive the probability 
of a biased step. It is not difficult to see
that the initial sampling variance indeed is 
$\Omega(n r^2)$; however, as the probabilistic model of the \rcga becomes 
more focused over time, the sampling variance 
will decrease and the probability of a biased
step will increase. This effect is ignored in this section. In the following section on 
\ronemax, we will analyze the reduction 
of the sampling variance over time.

We are now ready to state our first result
for \gonemax. 
For simplicity, it will consider the \rcga without borders. This 
allows us to assume \emph{well-behaved frequencies} \cite{Doerr-FOGA-19} here, which means that the starting 
frequency 
$1/r$ is a multiple of $1/K$ and hence, every frequency 
will only take values in $\{i/K \mid i=0,\dots,K\}$.
In the 
subsequent analysis on \ronemax, we begin by presenting our results without borders. In Section~\ref{sec:borders}, we consider the \rcga with borders where well-behaved frequencies are ruled out.

\begin{theorem}
\label{theorem:complexity-gonemax}
The runtime of the \rcga without borders on the \gonemax problem is $\bigo(Kr\sqrt{n} (\log r + \log K))$ with high probability, where $K\geq cr^2\sqrt{n}\log n$ for some sufficiently large constant $c>0$ and both $K$ and $r$ are polynomial functions of $n$. The runtime bound is $\bigo(nr^3\log^2 n \log r)$ for $K=r^2\sqrt{n}\log n$.
\end{theorem}

\begin{proof}
The main concept is to estimate the expected optimization time, assuming low genetic drift, by studying the process described by an arbitrary frequency~$p_{i,r-1}^{(t)}$ for value~$r-1$. We employ additive drift analysis over a sequence of phases. This is followed by multiplicative drift analysis. Specifically, additive drift analysis with tail bounds is applied over a sequence of phases from the initial frequency $1/r$ to $1/2$, while multiplicative drift analysis with tail bounds is used from $1/2$ to the maximum frequency.

Using the choice of $K\ge cr^2\sqrt{n}\ln n$, we apply Lemma~\ref{lemma:negativedrift-gonemax} to analyze and control genetic drift. This allows us to demonstrate that after $T=c'Kr\sqrt{n}(\log r+\log K)$ iterations, no frequency of value $r-1$ deviates by more than $1/(2r)$ in the negative direction from a given value it has reached. According to Lemma~\ref{lemma:negativedrift-gonemax}, the probability of a single frequency dropping by $1/(2r)$ is at most $n^{-K/(288r^2\sqrt{n}\ln n)} \le n^{-c/288}$. By choosing $c\ge 288\kappa$, the failure probability becomes $\bigo (1/n^{\kappa})$ for a self-chosen constant~$\kappa$. Furthermore, by applying the union bound over all frequencies, the probability of any frequency dropping by $1/(2r)$ remains bounded by $\bigo(1/n^{\kappa-1})$.

To reach the value $1/2$ starting from $1/r$, we need $r/2$ phases with an increase of $1/r$ per phase, starting from phase~1. We consider phase indices $k=1,2,\dots,r/2$.
Additionally, thanks to our analysis of genetic drift, we assume that throughout phase~$k$, the condition $p^{(t)}_{i,r-1}\ge k/r - 1/(2r)\ge k/(2r)$ holds for each $i\in \{1,\dots,n\}$. 
The aim is to use an additive drift theorem with tail bounds to limit the time $T_k$ to complete phase~$k$. Since the 
process involves many idle steps (as already mentioned, the frequency changes
with probability at most $2z_i$, where $z_i\coloneqq p_{i,r-1}^{(t)}(1-p_{i,r-1}^{(t)})$), the simple additive drift theorem with tail bounds \cite[Theorem 2]{Timo2016Algo} does not give a strong enough result. Instead, we use the stronger Theorem~15 from \cite{Timo2016Algo} that applies to processes whose increments are bounded by sub-Gaussian random variables. For notational convenience, let $X_t=p_{i,r-1}^{(t)}$. Let 
$\varepsilon=\expect{X_{t+1}-X_t\mid \filt}$ denote the drift, which, according to 
Lemma~\ref{lemma:sdrift} satisfies 
\begin{align*}
\varepsilon & \ge \frac{8z_i}{9K (2(r-1)\sqrt{3n}+1)}
 \ge \frac{8z_i}{20Kr \sqrt{3n}} \ge \frac{z_i}{3Kr \sqrt{3n}} \eqqcolon \varepsilon'.
\end{align*}
In the following, we pessimistically assume 
$\varepsilon=\varepsilon'$. To apply the drift theorem, 
the difference $X_{t+1}-X_t-\varepsilon$ has to be stochastically bounded. This is done below in Lemma~\ref{lemma:phidriftsubgaussian}, part~1, which shows that the difference is $(4z_i/K^2+2\varepsilon/K,K)$-sub-Gaussian. Finally, the 
time bound 
$s_k$ for the phase must satisfy $s_k\ge 2D/\varepsilon$, where $D=1/r$ is the distance to be overcome. We choose $s_k\coloneqq 2D/\varepsilon$. With these 
quantities, the drift 
theorem yields
 \begin{align*}
        & \prob{T_k \geq s_k} \leq \exp\left( -\frac{s_k\varepsilon}{4} \min \left( \delta, \frac{\varepsilon}{2c} \right) \right) \\ & \le 
       \exp\left(-\frac{D}{2} \min \left( K, \frac{z_i}{3Kr \sqrt{3n}} \frac{1}{8z_i/K^2+4\varepsilon/K} \right)\right).
    \end{align*}
    Since $8z_i/K^2+4\varepsilon/K = 8z_i/K^2 + 4z_i/(3K^2r\sqrt{3n}) \le 9z_i/K^2$, the above minimum is taken on the second term and we obtain
    \[
    \prob{T_k \geq s_k} \le 
    \exp\left(-\frac{1}{2r} \frac{K}{27r\sqrt{3n}}\right) \le 
    \exp\left(-\frac{c \log n}{54\sqrt{3}}\right),
    \]
    where the last inequality used the assumption $K\ge cr^2\sqrt{n}\log n$. Hence, $T_k<s_k$ \whp

Now, we apply a union bound over all $r/2$ phases, ensuring that we still have a high probability of each phase finishing within at most 
\begin{align*}
\frac{2D}{\varepsilon} & = \frac{2}{r}\cdot \frac{12Kr^2\sqrt{3n}}{k}  \le \frac{42Kr\sqrt{n}}{k}
\end{align*}
steps, where we 
used our assumption $p_{i,r-1}^{(t)}\ge k/(2r)$ and $1-p_{i,r-1}^{(t)}\ge 1/2$, so $z_i\ge k/(4r)$. Furthermore, by summing up all $s_k$, we can bound the total time for all phases $1,\dots,r/2$, which is $\bigo(Kr\sqrt{n}\log r)$ with probability at least $1-\bigO(rn^{-\kappa})$ for any constant $\kappa > 0$.

After the frequency reaches $1/2$, we need to analyze the remaining time until $p^{(t)}_{i,r-1}$ reaches its maximum (which is 1). In this phase, we apply the multiplicative drift with tail bounds as stated in  \cite[Theorem 2.4.5]{Lengler2020}. The starting point is $p^{(t)}_{i,r-1} \ge 1/2$, and by 
our analysis of genetic drift, the frequency will not decrease below $1/4$ \whp Therefore, we will assume $p^{(t)}_{i,r-1} \ge 1/4$. Let $q^{(t)}_{i,r-1} \coloneqq 1 - p^{(t)}_{i,r-1}$. Further, we bound $p^{(t)}_{i,r-1} (1-p^{(t)}_{i,r-1})$ by using $p^{(t)}_{i,r-1}\geq 1/4$ and let $  q^{(t)}_{i,r-1}\coloneqq 1 - p^{(t)}_{i,r-1}$. Then, by Lemma~\ref{lemma:sdrift}, we obtain:
\begin{align*}
\expect{q^{(t)}_{i,r-1} - q^{(t+1)}_{i,r-1}} &  \geq \frac{8\left(p^{(t)}_{i,j} (1 - p^{(t)}_{i,j})\right)}{9K \left(2(r-1)\sqrt{3n}+1\right)}
\geq \frac{8\cdot \frac{1}{4}\cdot q^{(t)}_{i,r-1}}{9K \left(2(r-1)\sqrt{3n}+1\right)} \\
& \geq q^{(t)}_{i,r-1}\cdot\frac{2}{9K \left(2(r-1)\sqrt{3n}+1\right)}.
\end{align*}

Assuming a well-behaved frequency, the smallest possible state is $1/K$, by applying the multiplicative drift theorem with tail bounds \cite[Theorem 2.4.5]{Lengler2020}, with $s_{\min}=1/K$ and $X_t=q^{(t)}_{i,r-1}$, and using the previously derived value of 
\[\delta = \frac{2}{9K \left(2(r-1)\sqrt{3n}+1\right)},\]
it follows that for $T\coloneqq \min\{t\mid X_t = 0\}$, the following holds:

\begin{align*}
\probbig{T \le \frac{u+\ln{(X_0/s_{\min})}}{\delta} }  
\le \probbig{T \le  \tfrac{(u+\ln K)\cdot 9K (2(r-1)\sqrt{3n}+1)}{2}} \le e^{-u}.
\end{align*}

By selecting $u=c\ln K$, where $c$ is a constant, we get 
\[T \le \frac{(c+1)\ln K}{2}\cdot 9K \left(2(r-1)\sqrt{3n}+1\right) \]
with probability at least $1-e^{-c\ln n}$. 

Therefore, the time for the frequency to reach its maximum~$1$ is $\bigo (\sqrt{n}rK\log K))$ with high probability. Further, we obtain that the total time until all frequencies for value $r-1$ have been raised to their maximum  is $\bigo (\sqrt{n}rK(\log r +\log K))$ with high probability by a union bound. At this point, the 
optimum $(r-1,\dots,r-1)$ will be sampled. By selecting $K\geq cr^2\sqrt{n}\log n$, which limits genetic drift as mentioned above, the runtime of the \rcga on \gonemax is $\bigo (nr^3\log^2 n\log r)$ with high probability.
\end{proof}

\subsection{Improved Runtime Analysis of the $r$-cGA on $r$-OneMax} 

The initial runtime analysis of the \rcga on \ronemax was independently provided by Adak and Witt and Hamano et al.\ \cite{AdakPPSN2024,HamanoECJ24}. We revisit this 
problem now and improve the previously best result by 
a logarithmic factor. This method applies to scenarios without frequency borders. However, unlike the previous analyses, the following Section~\ref{sec:borders} also presents a theorem that studies the above-mentioned borders. 
To prove our runtime bound, we will employ the negative drift theorem (see Theorem~\ref{theorem:negativedrift}) and additive drift analysis with tail bounds. This approach is different from \cite{HamanoECJ24}, where multiplicative drift on a non-linear potential function of frequencies of the kind
$(1-p_{i,r-1}^{(t)})/p_{i,r-1}^{(t)}$ was
used. Moreover, in contrast to~\cite{HamanoECJ24}, we do not only consider 
the drift of single frequencies, but the 
combined drift of frequencies, which allows 
us to track the sampling variance and to obtain
tight bounds for $r=O(1)$ (cf.~\cite{sudholt2019choice}).

\begin{theorem}
\label{theorem:complexity-ronemax}
\sloppy
The runtime of the \rcga without borders on the \ronemax problem is bounded by $\bigo(K\sqrt{n}\log r)$ with high probability. In this expression, $K$ is at least $cr\sqrt{n}\log n$, where $c>0$ is a sufficiently large constant, and both $K$ and $r$ are polynomial functions of $n$. For $K=cr\sqrt{n}\log n$, the runtime bound is $\bigo(nr\log n \log r)$.
\end{theorem}

To prove the aforementioned theorem on runtime bounds, we require the following lemmas, which have been taken from~\cite{AdakPPSN2024}, with slight modifications to the notation. In the first lemma, we establish a positive drift  
of an individual frequency. In the second lemma, we capture the combined drift of  frequencies within a potential function~$\varphi_t$.

\begin{lemma}
\label{lemma:sdrift-ronemax}
 If  $\frac{1}{K} \leq p^{(t)}_{i,r-1} \leq 1 - \frac{1}{K}$, then   
\[
\E(\Delta_{i,r-1}\mid p^{(t)}_{i,r-1}) \ge \frac{8}{9}\cdot\frac{p^{(t)}_{i,r-1} (1 - p^{(t)}_{i,r-1})}{K(2\sqrt{3V}+1)},
\]
where $V=\sum_{i=1}^n p^{(t)}_{i,r-1} (1 - p^{(t)}_{i,r-1})$.
\end{lemma}

\begin{lemma}
\label{lemma:drift-of-phi}
For any~$t\ge 0$, let $\varphi_t \coloneqq \sum^{n}_{i=1} (1 - p^{(t)}_{i,r-1})$.
If, for any $t\ge 0$, it holds that 
$p_{i,r-1}^{(t)} \ge 1/4$ for all $i\in\{1,\dots,n\}$  and furthermore $\varphi_t\ge 1/2$, then 
\[
\E(\varphi_t - \varphi_{t+1} \mid \varphi_t) \ge \frac{\sqrt{\varphi_t}}{30K}.
\]
\end{lemma}

\begin{proofof}{Theorem~\ref{theorem:complexity-ronemax}}
We show that, starting with a setting where all frequencies are at least $1/r$, after $\bigo(K\sqrt{n}\log r)$ iterations, \whp, the optimum is found, and no frequency has ever dropped below $1/(2r)$. The key approach involves performing an additive drift analysis with tail bounds in several phases to bound the expected time until all $p^{(t)}_{i,r-1}$ have been at least $1/2$ at 
least once by $\bigo(K\sqrt{n}\log r)$ with high probability. 
This analysis is largely 
identical to the corresponding part in the 
proof of Theorem~\ref{theorem:complexity-gonemax}. 
Once we reach a value of $1/2$, we perform a different additive drift analysis with tail bounds in several phases, no longer based on individual frequencies but on a potential function involving all frequencies. At this point, our analysis is somewhat similar to a related analysis by Sudholt and Witt~\cite[Theorem 5]{sudholt2019choice} for the binary \cga with borders. At 
the same time, our analysis is
simpler since it avoids the borders and uses additive 
drift analysis with tail bounds 
instead of the more involved 
variable drift with tail bounds carried out in 
\cite{sudholt2019choice}.

Using the choice of $K\ge cr\sqrt{n}\ln n$, we use Lemma~\ref{lemma:negativedrift-ronemax} to analyze and control genetic drift. This method allows us to show that after $T=c'K\sqrt{n}\log r$ steps, any  frequency for value $r-1$ will not deviate by more than $1/(2r)$ in the negative direction  with high probability. According to Lemma~\ref{lemma:negativedrift-gonemax}, the probability of a single frequency dropping by $1/(2r)$ is at most $n^{-K/(288r\sqrt{n}\ln n)} \le n^{-c/288}$. By selecting 
$c \ge 288\kappa$, the failure probability is bounded by $\bigo(1/n^{\kappa})$ for any constant $\kappa$. Furthermore, by applying the union bound across all frequencies, the probability of any frequency dropping by $1/(2r)$ remains bounded by $\bigo(1/n^{\kappa-1})$.

The progress of a frequency is analyzed in phases, with each phase increasing the frequency by $1/r$. This process continues until the frequency reaches $1/2$. The total number of phases required to reach $1/2$ from $1/r$ is $r/2$, with phases indexed as $k = 1, 2, \dots, r/2$.

Furthermore, thanks to our analysis of genetic drift, we assume that throughout phase $k$, the condition $p^{(t)}_{i,r-1}\ge (k/r-1/(2r))\ge k/(2r)$ holds for position~$i$. The aim is to use an additive drift theorem with tail bounds to limit the time $T_k$ to complete phase~$k$. The 
analysis is largely 
identical to the corresponding 
part in Theorem~\ref{theorem:complexity-gonemax}, with the 
only exception that 
here, 
$\varepsilon$ and $K$ are 
different, more precisely 
by an asymptotic factor $r$ 
bigger and smaller, respectively. 
Again, by \cite[Theorem 2]{Timo2016Algo} and Lemma~\ref{lemma:phidriftsubgaussian}, we obtain 
 \begin{align*}
        & \prob{T_k \geq s_k} \leq \exp\left( -\frac{s_k\varepsilon}{4} \min \left( \delta, \frac{\varepsilon}{2c} \right) \right) \\ & \le 
       \exp\left(-\frac{D}{2} \min \left( K, \frac{z_i}{3K \sqrt{n}} \frac{1}{8z_i/K^2+4\varepsilon/K} \right)\right)    
        .
    \end{align*}
    Now, we substitute the present value for $\varepsilon$ from Lemma~\ref{lemma:sdrift-ronemax}, which, after straightforward simplifications, gives 
    $\varepsilon\ge \frac{z_i}{3K\sqrt{n}}$. 
     We pessimistically assume this drift bound 
    to hold with equality. 
    Thereby, we estimate  $8z_i/K^2+4\varepsilon/K = 8z_i/K^2 + 4z_i/(3K^2\sqrt{n}) \le 9z_i/K^2$. Hence, the above minimum is again taken on the second term. We obtain 
    \[
    \prob{T_k \geq s_k} \le 
    \exp\left(-\frac{1}{2r} \frac{K}{27\sqrt{n}}\right) \le 
    \exp\left(-\frac{c \log n}{54}\right),
    \]
    where the last inequality used the assumption $K\ge cr\sqrt{n}\log n$. Hence, $T_k<s_k$ \whp

Next, we apply a union bound over all $r/2$ phases, which ensures that each phase completes with a high probability, with the total time not exceeding:
\begin{align*}
\frac{2D}{\varepsilon} & \le  \frac{2}{r}\cdot \frac{9rK(2\sqrt{n}+1)}{2k} 
 \le \frac{18K\sqrt{n}}{k},
\end{align*} 
where  we 
used our assumption $p_{i,r-1}^{(t)}\ge k/(2r)$ and $1-p_{i,r-1}^{(t)}\ge 1/2$, so $z_i\ge k/(4r)$.
By summing over all $s_k$, we obtain $\bigo(K\sqrt{n}\log r)$ on the total time for all phases $1,\dots,r/2$. By a union bound, this bound also holds for all frequencies $p_{i,r-1}$ together \whp

Once a frequency has reached $1/2$, we need to analyze the remaining time until it reaches its maximum value~$1$. During this stage, we will again employ additive drift analysis with tail bounds in the strong version. Our analysis requires that all frequencies satisfy $p^{(t)}_{i,r-1}\ge 1/4$, which holds \whp due to our analysis of genetic drift. Now we consider the potential function $\varphi_t\coloneqq\sum_{i=1}^{n}(1-p^{(t)}_{i,r-1})$, which combines all frequencies into a single value, and then track this potential.

We consider a number of phases such that the potential is consecutively halved. More precisely, phase~$i$ starts with a potential of at most $n/2^{i-1}$ and ends when the potential has become at most $n/2^{i}$. Since the potential is at least $n/2^i$ in the phase, then according to Lemma~\ref{lemma:drift-of-phi}, the drift in the phase is bounded according to
\[
\E(\varphi_t - \varphi_{t+1} \mid \varphi_t) \ge \frac{\sqrt{n/2^i}}{30K}.
\]

We again apply the additive drift theorem with tail bounds for drift processes with sub-Gaussian increments \cite[Theorem 15]{Timo2016Algo}, this time 
on the process $\varphi_t$. 
Using  Lemma~\ref{lemma:phidriftsubgaussian}, part~2, we have $\delta = K$, and $c=8\varphi_t/K^2$ and $s_i= 2D/\varepsilon$ where $D$ is 
the distance to be bridged 
and $\varepsilon=\sqrt{\varphi_t}/(30K)$. To handle 
a possible overshooting of the target $n/2^i$, we replace it by the most remote state possible, which is~$0$, and set $D=(n/2^{i-1} - 0)=n/2^{i-1}$ despite the fact that we are only interested in a distance 
of $n/2^{i-1} - n/2^{i} = n/2^i$. 
Applying \cite[Theorem~15]{Timo2016Algo}, we obtain for 
the time~$T_i$ to complete phase~$i$ that 
\begin{align*}
\prob{T_{i}\ge s_i} & \le \exp\left( -\frac{D}{2} \min\left(K,\frac{K}{480\sqrt{\varphi_t}}\right) \right).
\end{align*}

Now, by substituting our assumption that $K\ge cr\sqrt{n}\log n$ and $\varphi_t\ge \sqrt{n/2^i}$, we conclude
\begin{align*}
\prob{T_{i}\ge s_i} & \le \exp\left( - \frac{c'nr\log n}{\sqrt{2^i}} \right),
\end{align*}
where $c'$ is a constant.
Hence, the time required to complete the phase~$i$ is at most
\[\frac{2n}{2^{i-1}}\cdot\frac{30K}{\sqrt{n/2^{i-1}}} \le 60K\sqrt{n/2^{i-1}}\]
\whp

We choose constants in the time bound appropriately to obtain a failure probability of at most $1/n^\kappa$ per phase for any constant~$\kappa>0$. We consider at most $\log n$ phases until the potential has dropped
to at most 1/2. The total time spent in these phases is then at most
\[
\sum_{i=1}^{\log n} \frac{60K\sqrt{n}}{\sqrt{2^{i-1}}} = \bigo(K\sqrt{n}),
\] which holds \whp by 
a union bound over the $\bigo(\log n)$ phases.


Altogether, with high probability, the potential has reached at most $1/2$ within
$\bigo(K\sqrt{n})$ iterations. After this point, the probability of sampling the optimum becomes a constant. If the optimum is not found, the expected next value of the potential is no more than $1/2 + n/K \le \sqrt{n}$, since in each step, the frequency decreases by at most $1/K$. With high probability, the time to return to a potential of at most $1/2$ is $\bigo(K\sqrt{\sqrt{n}}) = \bigo(Kn^{1/4})$, as derived from the previous analysis of the drift of the potential. By repeating this argument $C\log n$ times, after $\bigo(K (\log n) n^{1/4})$ additional steps, the optimum will be sampled with probability at least $1-e^{\Omega(C\log n)}$, which ensures that the optimum is found with high probability.

By adding the two stages, the total time becomes $\bigo (K\sqrt{n}\log r)$ with high probability. Therefore, by choosing $K  =  cr\sqrt{n}\log n$, the runtime of the \rcga on the \ronemax problem is $\bigo(nr\log n\log r)$  with high probability.
\end{proofof} 

We still have to prove the outstanding claim on the increments 
analyzed in the two applications of the additive drift theorem with tail bounds. 
This requires a closer look at the stochastic behavior of
a frequency and of the potential function $\varphi_t = \sum_{i=1}^n (1-p_{i,r-1}^{(t)})$, as done in the following lemma. 

\begin{lemma}\label{lemma:phidriftsubgaussian}
For any $i\in\{1,\dots,n\}$, let $X_t\coloneqq p_{i,r-1}^{(t)}$  and let $\varepsilon_i=  
\expect{X_{t+1}-X_t\mid \filt}$, where $\filt$ is 
a filtration with respect to the underlying stochastic process. Then the random 
variable $X_{t+1}-X_t-\varepsilon_i$ is 
$(4X_t(1-X_t)/K^2 + 2 \varepsilon_i/K,K)$-sub-Gaussian in the sense 
    of \cite{Timo2016Algo}.

    Let $\varepsilon_t \ge  \sqrt{\varphi_t}/(30K)$ 
    as derived in Lemma~\ref{lemma:drift-of-phi}. Then the random variable 
    $\varphi_{t} - \varphi_{t+1} - \varepsilon_t$ 
    is $(8\varphi_t/K^2,K)$-sub-Gaussian. 
\end{lemma}

\begin{proof}
We start with the first claim. 
We abbreviate $p_i \coloneqq p_{i,r-1}^{(t)}$, 
$z_i\coloneqq p_i(1-p_i)$ and $\delta_i \coloneqq p_{i,r-1}^{(t+1)}- p_{i,r-1}^{(t)}$. We re-use 
the characterization of the $\delta_i$  
from Section~\ref{section:framework}, 
which says that the 
frequency~$p_i$ changes with probability 
$2z_i$ by either $+1/K$ or 
$-1/K$, and, thanks to the 
b-steps, has a slight drift 
towards increasing values:
\[
\delta_i = \begin{cases}
    +1/K & \text{ with prob. $z_i(1+K\varepsilon_i/(2z_i))$}, \\
    -1/K & \text{ with prob. $z_i(1-K\varepsilon_i/(2z_i)$)},\\
    0 & \text{ with the remaining 
    probability,}
    \end{cases}
\]
which gives $\expect{\delta_i}=\varepsilon_i$. 

The aim is to bound the moment-generating function (mgf.) of the differences 
$\delta_i-\varepsilon_i$. Since $\varepsilon_i$ 
is a scalar, we first bound the mgf. of 
$\delta_i$. 
We have for $\lambda\ge 0$ that  
\begin{align*}
& \expect{e^{\lambda \delta_i}} \\ 
& 
= z_i (1+K\varepsilon_i/(2z_i)) e^{\lambda/K} + 
z_i (1-K\varepsilon_i/(2z_i)) 
e^{-\lambda/K} + 
 (1-2z_i) e^{\lambda 0} \\
& = 2z_i \cosh(\lambda/K)   
+ K\varepsilon_i \sinh(\lambda/K) + 
(1-2z_i). 
\end{align*}
We use the following Taylor expansions for the 
hyperbolic functions: $\cosh(x)=1+x^2/2!+x^4/4!+\dots$ and $\sinh(x)=x  +x^3/3! +x^5/5! + \dots$. Thereby, we obtain for $\lambda/K\le 1$ that 
\begin{align*}
\expect{e^{\lambda \delta_i}} & \le 
2z_i(1+(\lambda/K)^2) + K\varepsilon_i (\lambda/K + (\lambda/K)^3) + 1-2z_i \\
& \le e^{2z_i (\lambda/K)^2 + \lambda \varepsilon_i + (\lambda/K)^2 \lambda \varepsilon_i } ,
\end{align*}
so 
\[
\expect{e^{\lambda \delta_i-\lambda \varepsilon_i}}  
\le 
e^{\frac{\lambda^2}{2} \frac{(4z_i + 2\lambda \varepsilon_i )}{K^2}} \le  
e^{\frac{\lambda^2}{2} (\frac{4z_i}{K^2} + \frac{ 2 \varepsilon_i }{K})},
\]
using $\lambda\le K$, 
meaning that the random variable $\delta_i-\varepsilon_i$ is $(4z_i/K^2 + 2 \varepsilon_i/K,K)$-sub-Gaussian in the sense of \cite{Timo2016Algo}. 

For the second statement, we decompose $\Delta_t = \varphi_{t}-\varphi_{t+1}$ 
into 
the sum of independent increments 
$\delta_i$ already analyzed 
above, \ie, 
$\Delta_t = \sum_{i=1}^n \delta_i$. Note also that 
$\varepsilon_t=\sum_{i=1}^n \varepsilon_i$. To analyze 
the mgf. of $\Delta_t-\varepsilon_t$, 
we sum over 
the $n$ differences in 
the exponent of the mgf.
Trivially bounding 
$z_i\le 1-p_i$, we have 
\[
\expect{e^{\lambda (\Delta_t-\varepsilon_t)}}
= \expect{e^{\lambda \sum_{i=1}^n (\delta_i - \varepsilon_i)}} 
\le 
e^{\frac{\lambda^2}{2} (4\varphi_t / K^2 + 
2\varepsilon_ t /K)}.
\]
With our bound $\varepsilon_t \ge  \sqrt{\varphi_t}/(30K)$ and bounding $\sqrt{\varphi_t}\le \varphi_t$, we have 
that the difference is $(8\varphi_t/K^2,K)$-sub-Gaussian.
\end{proof}

\subsection{Analysis in the Presence of Borders}
\label{sec:borders}

The \rcga with borders is harder to analyze 
than without borders since frequencies that have reached the upper border~$1-1/n$ may 
decrease again in subsequent steps. 
Previous work on the binary \cga (e.\,g. \cite{sudholt2019choice}) addresses this 
by bounding the number of frequencies 
that leave the upper border again via 
a non-trivial drift analysis. Moreover, the
previous analyses are crucially based 
on the well-behaved frequencies assumption: 
in particular, this assumption is used to avoid 
uneven step sizes when taking the last 
step towards the maximum frequency. If, e.\,g., the second-largest frequency value 
was strictly in between $1-1/n-1/K$ and 
$1-1/n$, then the possible decrease would
be $1/K$ but the increase strictly less. In 
this situation, Lemma~\ref{lemma:drift-of-phi} would not hold and the frequency 
could exhibit a drift away from its maximum. The well-behaved frequency assumption 
(used in \cite{sudholt2019choice,doerr2021runtime} and other works) postulates 
that $K$ is chosen so that it evenly divides
the interval of allowed frequencies.

Unfortunately, the \rcga for $r\ge 3$ includes a non-trivial procedure 
for frequency capping that may result 
in step sizes different from $1/K$ even 
for frequencies that are far away from a border. More precisely, if, for some position~$i$,
the two values sampled are $a$ and $b$, where 
$a,b\in\{0,\dots,r-1\}$, $a\neq b$ 
and at least one of $p_{i,a}$, $p_{i,b}$ 
is at the lower border, then probability 
may be taken also from frequencies $p_{i,c}$ for $c\notin\{a,b\}$ to normalize the vector. 
See \cite{BENJEDIDIA2024114622} for details. 
The amount by which the other frequencies 
is reduced cannot necessarily be expressed 
in units (or fixed subunits) of $1/K$. However, 
as already observed in \cite{AdakWittArxiv2025}, the update is 
reduced by no more than $1/((r-1)n)$, which 
is the value of the lower frequency border.
In other words, any positive frequency update 
is at least $1/K-1/((r-1)n)$ unless it 
is capped at the upper frequency border.

The paper \cite{AdakWittArxiv2025} was 
the first to analyze the \rcga in the presence 
of frequency borders. Since their study 
is on the \textsc{LeadingOnes} function, 
the optimal value of~$K$ is different 
from the present paper, more precisely, 
they choose $K=\Omega(n\log n)$, while 
the results above use $K=\Omega(\sqrt{n}\log n)$. The larger value of $K=\Omega(n\log n)$ comes in handy, since a frequency value of $1-1/n-1/K$, which 
still allows a full update of $1/K$ towards 
the upper border, is essentially as good 
as the value $1-1/n$  self; note that 
$1-1/n-1/K = 1-1/n-o(1/n)$ for this choice 
of~$K$. We can use the same insights 
to analyze the \rcga with borders on \ronemax for 
$K=\Omega(n\log n)$, see Theorem~\ref{theo:ronemaxwithborders} below; however, this 
does not give the best possible runtime 
result.

\begin{theorem}
    \label{theo:ronemaxwithborders}
    Let $K\ge cnr\log n$ for a sufficiently 
    large constant~$c>0$ and 
    assume that $K=o(n^2)$ and $r=\mathit{poly}(n)$. 
    Then the runtime of the \rcga with borders on the \ronemax problem is bounded by $\bigo(K\sqrt{n}\log r)$ with high probability. For $K=cnr\log n$, the runtime bound is $\bigo(n^{3/2}r\log n \log r)$.
\end{theorem}

\begin{proof}
The proof follows largely the structure of the 
proof of the case without borders in Theorem~\ref{theorem:complexity-ronemax}. The first 
stage is completely identical since it studies 
the time until all frequencies for value~$r-1$ 
have risen to at least~$1/2$, and this process 
cannot be negatively affected by the lower borders 
on frequencies. Hence, with high probability 
all frequencies will have reached~$1/2$ at least 
once after $\bigo(K\sqrt{n}\log r)$ steps. Also 
with high probability by the negative drift theorem,
they will not drop below~$1/4$ in the next $c_1K\sqrt{n}$ 
steps, where~$c_1$ is an arbitrary constant.

Also the analysis of the second stage, consisting 
of $\bigo(K\sqrt{n})$ iterations, is largely identical,
except that we apply Lemma~\ref{lemma:drift-of-revised-phi} instead of Lemma~\ref{lemma:drift-of-phi} and stop the iterated additive drift 
analysis at the first point where $\varphi_t<10000$. 
As argued in 
\cite[Proof of Theorem~2]{sudholt2019choice}, 
  the probability 
of sampling the optimum at this point is 
$\Omega(1)$ since it is minimized 
if as many frequencies as possible take the 
extreme values~$1-1/n-1/K$ and $1/4$, 
using arguments related to 
majorization and Schur-convexity. After reaching 
potential at most~$10000$, we consider the 
a phase consisting of the next $T_2\coloneqq c_2\log n$ iterations for an arbitrary constant~$c_2$. 
By our assumption on $K\ge cnr\log n$, each iteration 
decreases a frequency by $\bigo(1/(n\log n)$, so 
throughout the phase of $T_2$ steps all frequencies 
that were at least $1-1/n-1/K$ at the beginning
of the phase 
will still be $1-O(1/n)$. Therefore, the probability 
of sampling the optimum is $\Omega(1)$ in every of 
the $T_2$ steps, and the probability of not sampling 
it in the phase is at most $e^{-\Omega(c_2\log n)}$. Hence, 
altogether  the optimum is sampled in $\bigo(K\sqrt{n}\log r)$ iterations with high probability.
\end{proof}

To study the progress of the frequency vectors, we define a potential function similar to the one used above 
in the proof of Theorem~\ref{theorem:complexity-ronemax} and in previous 
works \cite{sudholt2019choice,doerr2021runtime,AdakPPSN2024}. However, 
we already consider all frequencies above $1-1/n-1/K$ as optimal.

\begin{definition}
    Given the frequency vector 
    $(p_{1,r-1}^{(t)},\dots, 
    p_{n,r-1}^{(t)})$ for value~$r-1$ in the \rcga at time~$t$, hereinafter abbreviated  
    as $(p_1,\dots,p_n)$, we define
    \[
    q_i = \begin{cases}
    0 & \text{if $p_i\ge 1-1/n-1/K$}\\
    1-1/n-1/K- p_i & \text{otherwise}
\end{cases}
   \]
   for $i\in\{1,\dots,n\}$ 
   and let $\varphi_t = \sum_{i=1}^n q_i$. 
\end{definition}

Also, similarly to the  previous  Lemma~\ref{lemma:drift-of-phi},  
the drift of this potential function is proportional 
to its square root.

\begin{lemma}
\label{lemma:drift-of-revised-phi}
Let $K=\Omega(nr\log n)$ and $K=o(n^2)$. If 
for all $i\in\{1,\dots,n\}$ it holds that 
$p_{i,r-1}^{(t)} \ge 1/4$ and furthermore $\varphi_t\ge 10000$, then for $n$ large enough it holds that 
\[
\E(\varphi_t - \varphi_{t+1} \mid \varphi_t) \ge \frac{\sqrt{\varphi_t}}{66K}.
\]
\end{lemma}

\begin{proof}
We look into the current frequency vector $(p_1,\dots,p_n)$ 
    for value~$r-1$ and distinguish between the 
    ones above~$1-1/n-1/K$ (\ie, essentially at the upper border) and the remaining ones. Without loss of generality, for some 
    $s\in\N_0$, the frequencies $p_1,\dots,p_{s}$ are all 
    at least $1-1/n-1/K$ . Each of these~$s$ frequencies changes 
    its value with probability at most~$2/n$ since it is necessary
    to sample a value different from~$r-1$ at the position 
    in at least one offspring. The maximum change of the 
    underlying $q_i$ as defined above is $2/K$ (using $K=\omega(n)$), so 
    the expected increase of the potential due to the 
    first $s$~frequencies  is at most $\frac{4s}{nK} \le \frac{4}{K}$. More formally, 
    $\sum_{i=1}^s (p_{i,r-1}^{(t)} - p_{i,r-1}^{(t+1)})\le 4/K$.

    Each of the other frequencies can both decrease 
    and increase by~$1/K$. Fix 
    an arbitrary $i>s$ and 
    assume that $p_i \le 1-1/n-1/K$. For 
    the moment, we assume that the 
    positive update is exactly~$1/K$, which happens if no frequency $p_{i,j}$ 
    for $j<r-1$ is capped at the lower 
    border. 
    Following the proof of Lemma~\ref{lemma:sdrift} with 
    $\Delta_i = p_{i,r-1}^{(t)} - p_{i,r-1}^{(t+1)}$ 
    we have 
    for that 
    \[
    \expect{\Delta_{i} \mid p_{i}} 
    \ge 
    \frac{8p_i(1-p_i)}{9K(2\sqrt{V}+1)},
    \]
    where $V=\sum_{j=1}^n p_j(1-p_j)$. By definition, we have $1-p_i \le q_i + 1/n + 1/K \le q_i + 2/n$ by our assumption on~$K$. Hence,  
    $V \le \sum_{j=1}^n q_j +2 = \varphi_t +2$, and, 
    since $\varphi_t \ge 10000$, we estimate 
    $2\sqrt{V} + 2 \le 4\varphi_t$. Along with $p_i\ge 1/4$, 
    we arrive at the bound 
    \[
    \expect{\Delta_{i} \mid p_{i}}  
    \ge 
    \frac{1-p_i}{
    18K\sqrt{\varphi_t}},
    \]
    Now, since $1-p_i \ge q_i$ and $\varphi_t =\sum_{i=s+1}^n q_i$ (since the first~$s$ positions 
    do not contribute to~$\varphi_t$), we have  
    \begin{align*}
    \expect{\varphi_t - \varphi_{t+1} \mid \varphi_t}  
    & \ge 
    \sum_{i=s+1}^n \frac{q_i}{18K\sqrt{\varphi_t}} - \sum_{i=1}^s (p_{i,r-1}^{(t)} - p_{i,r-1}^{(t+1)})\\
    & = \frac{\sqrt{\varphi_t}}{18K} - \frac{4}{K} \ge 
    \frac{\sqrt{\varphi_t}}{65K} ,
    \end{align*}
    where the last inequality used $\varphi_t\ge 10000$.

    Finally, we correct the previous 
    drift bound by the 
    effect of~$p_i$ being
    decreased since frequencies 
    for other values than~$r-1$ are 
    capped at the lower border. This happens only with probability at most 
    $2(r-1)(\tfrac{1}{(r-1)n}+\tfrac{1}{K})\le 2/n + O(1/(n\log n)) \le 3/n$ since the 
    largest frequency value whose update may be capped at the lower border is $\frac{1}{(r-1)n} + \frac{1}{K}$, two offspring are sampled and~$r-1$ different values 
    are considered; note also that 
    $K=\Omega(nr\log n)$. Together 
    with the fact that the frequency in such a case 
    is reduced by at most $1/((r-1)n)$, the drift bound 
    is still at least
    \[
     \frac{\sqrt{\varphi_t}}{65K} - \frac{3}{(r-1)n^2} \ge 
     \frac{\sqrt{\varphi_t}}{66K} 
    \]
    for $n$ large enough since $K=o(n^2)$.
\end{proof}

We believe that the statement 
of Theorem~\ref{theo:ronemaxwithborders}
also holds for the choice 
$K\ge cnr\log n$ taken in 
Theorem~\ref{theorem:complexity-ronemax}. 
Proving or disproving this conjecture is an open problem.

\section{Analysis through Experimentation}
\label{section:experiment}

In this section, we present the results of experiments conducted to evaluate performance of the proposed algorithm without border restriction. The algorithm is implemented using the $C$ programming language, with the \text{WELL1024a} random number generator. In the experiments, we ran the \rcga on \ronemax and \gonemax for different $n$ and $K$. We compare the results for different values of $n\in\{100,\dots,400\}$ and present the average number of iterations for a range of hypothetical population sizes $K\in\{200,\dots,1000\}$. Various plots for $r\in\{3,4,5\}$ are shown in Fig.~\ref{figure:rcGAn100-400} for both \ronemax and \gonemax.

\begin{figure}
	\centering
\begin{tikzpicture}[scale=0.5]
\begin{axis}[
    xlabel={Hypothetical Population Size (K)},
    ylabel={Avg. Number of Iteration},
    xmin=200, xmax=1000,
    ymin=1000, ymax=50000,
    legend pos=outer north east,
    legend cell align=left,
    ymajorgrids=true,
    grid style=dashed,
]
   \addplot[
    color=red,
    mark=dot, very thick,
    ]
    table[ignore chars={(,)},col sep=comma] {Data/ronemax-n100-r3.txt};
    
    \addplot[
    color=blue,
    mark=dot, very thick,
    ]
    table[ignore chars={(,)},col sep=comma] {Data/ronemax-n200-r3.txt};
    
    \addplot[
    color=green,
    mark=dot, very thick,
    ]
    table[ignore chars={(,)},col sep=comma] {Data/ronemax-n300-r3.txt};
    
     \addplot[
    color=black,
    mark=dot, very thick,
    ]
    table[ignore chars={(,)},col sep=comma] {Data/ronemax-n400-r3.txt};
\end{axis}
\end{tikzpicture}
\begin{tikzpicture}[scale=0.5]
\begin{axis}[
    xlabel={Hypothetical Population Size (K)},
    ylabel={Avg. Number of Iteration},
    xmin=200, xmax=1000,
    ymin=1000, ymax=50000,
    legend pos=outer north east,
    legend cell align=left,
    ymajorgrids=true,
    grid style=dashed,
]
 
   \addplot[
    color=red,
    mark=dot, very thick,
    ]
    table[ignore chars={(,)},col sep=comma] {Data/ronemax-n100-r4.txt};
    
     \addplot[
    color=blue,
    mark=dot, very thick,
    ]
    table[ignore chars={(,)},col sep=comma] {Data/ronemax-n200-r4.txt};
    
    \addplot[
    color=green,
    mark=dot, very thick,
    ]
    table[ignore chars={(,)},col sep=comma] {Data/ronemax-n300-r4.txt};
    
     \addplot[
    color=black,
    mark=dot, very thick,
    ]
    table[ignore chars={(,)},col sep=comma] {Data/ronemax-n400-r4.txt};

\end{axis}
\end{tikzpicture}
\begin{tikzpicture}[scale=0.5]
\begin{axis}[
    xlabel={Hypothetical Population Size (K)},
    ylabel={Avg. Number of Iteration},
    xmin=200, xmax=1000,
    ymin=1000, ymax=50000,
    legend pos=outer north east,
    legend cell align=left,
    ymajorgrids=true,
    grid style=dashed,
]
   \addplot[
    color=red,
    mark=dot, very thick,
    ]
    table[ignore chars={(,)},col sep=comma] {Data/ronemax-n100-r5.txt};
    
   \addplot[
    color=blue,
    mark=dot, very thick,
    ]
    table[ignore chars={(,)},col sep=comma] {Data/ronemax-n200-r5.txt};
    
    \addplot[
    color=green,
    mark=dot, very thick,
    ]
    table[ignore chars={(,)},col sep=comma] {Data/ronemax-n300-r5.txt};
    
    \addplot[
   color=black,
   mark=dot, very thick,
   ]
   table[ignore chars={(,)},col sep=comma] {Data/ronemax-n400-r5.txt};

   \legend{$n=100$, $n=200$, $n=300$, $n=400$}    
\end{axis}
\end{tikzpicture}
\begin{tikzpicture}[scale=0.5]
\begin{axis}[
    xlabel={Hypothetical Population Size (K)},
    ylabel={Avg. Number of Iteration},
    xmin=200, xmax=1000,
    ymin=1000, ymax=50000,
    legend pos=outer north east,
    legend cell align=left,
    ymajorgrids=true,
    grid style=dashed,
]
   \addplot[
    color=red,
    mark=dot, very thick,
    ]
    table[ignore chars={(,)},col sep=comma] {Data/gonemax-n100-r3.txt};
    
    \addplot[
    color=blue,
    mark=dot, very thick,
    ]
    table[ignore chars={(,)},col sep=comma] {Data/gonemax-n200-r3.txt};
    
    \addplot[
    color=green,
    mark=dot, very thick,
    ]
    table[ignore chars={(,)},col sep=comma] {Data/gonemax-n300-r3.txt};
    
    \addplot[
   color=black,
   mark=dot, very thick,
   ]
   table[ignore chars={(,)},col sep=comma] {Data/gonemax-n400-r3.txt};

\end{axis}
\end{tikzpicture}
\begin{tikzpicture}[scale=0.5]
\begin{axis}[
    xlabel={Hypothetical Population Size (K)},
    ylabel={Avg. Number of Iteration},
    xmin=200, xmax=1000,
    ymin=1000, ymax=50000,
    legend pos=outer north east,
    legend cell align=left,
    ymajorgrids=true,
    grid style=dashed,
]
 
   \addplot[
    color=red,
    mark=dot, very thick,
    ]
    table[ignore chars={(,)},col sep=comma] {Data/gonemax-n100-r4.txt};
    
     \addplot[
    color=blue,
    mark=dot, very thick,
    ]
    table[ignore chars={(,)},col sep=comma] {Data/gonemax-n200-r4.txt};
    
    \addplot[
    color=green,
    mark=dot, very thick,
    ]
    table[ignore chars={(,)},col sep=comma] {Data/gonemax-n300-r4.txt};
   
    \addplot[
   color=black,
   mark=dot, very thick,
   ]
   table[ignore chars={(,)},col sep=comma] {Data/gonemax-n400-r4.txt};
     
\end{axis}
\end{tikzpicture}
\begin{tikzpicture}[scale=0.5]
\begin{axis}[
    xlabel={Hypothetical Population Size (K)},
    ylabel={Avg. Number of Iteration},
    xmin=200, xmax=1000,
    ymin=1000, ymax=50000,
    legend pos=outer north east,
    legend cell align=left,
    ymajorgrids=true,
    grid style=dashed,
]

   \addplot[
    color=red,
    mark=dot, very thick,
    ]
    table[ignore chars={(,)},col sep=comma] {Data/gonemax-n100-r5.txt};
    
   \addplot[
    color=blue,
    mark=dot, very thick,
    ]
    table[ignore chars={(,)},col sep=comma] {Data/gonemax-n200-r5.txt};
    
   \addplot[
   color=green,
   mark=dot, very thick,
   ]
   table[ignore chars={(,)},col sep=comma] {Data/gonemax-n300-r5.txt};
    
    \addplot[
   color=black,
   mark=dot, very thick,
   ]
   table[ignore chars={(,)},col sep=comma] {Data/gonemax-n400-r5.txt};
    
\legend{$n=100$, $n=200$, $n=300$, $n=400$}    
\end{axis}
\end{tikzpicture}
\caption{\textmd{Empirical runtime of the \rcga on \ronemax and \gonemax; with the top side showing results for \ronemax at $r=3$ (left), $r=4$ (middle), and $r=5$ (right), and the bottom side for \gonemax at $r=3$ (left), $r=4$ (middle), and $r=5$ (right); $n\in\{200,\dots,400\}$, $K\in\{100,\dots,1000\}$ and averaged over 100 runs.}}
\label{figure:rcGAn100-400}
\end{figure}

The empirical runtime initially starts at a high value, decreases to a minimum, and then increases again as $K$ continues to grow. Analyzing the results, this empirical study clearly demonstrates how runtime is influenced by the value of $r$ in both functions. 

Additionally, we conducted an experiment to represent the frequencies of a specific position across some values of $r$. Fig.~\ref{figure:freqplot} illustrates how the frequency evolves from its initial value until the optimal solution is attained for both \ronemax and \gonemax problems. The frequency values are color-coded, transitioning from blue (indicating low values) to red (indicating high values) as the iterations progress. This color gradient visually captures the dynamics of how the algorithm converges over time. Furthermore, it is evident how the frequency for $r-1$ ultimately reaches its maximum.

\begin{figure}
\centering
\begin{tikzpicture}[scale=0.65]
  \begin{axis}[grid=both,xlabel=$Iterations$,ylabel=$r$-$value$,zlabel=$frequency$]
    \addplot3+[only marks,scatter] table[col sep=comma] {Data/ronemaxfreq-n400-K600-r8.txt};
  \end{axis}
\end{tikzpicture}
\begin{tikzpicture}[scale=0.65]
  \begin{axis}[grid=both,xlabel=$Iterations$,ylabel=$r$-$value$,zlabel=$frequency$]
    \addplot3+[only marks,scatter] table[col sep=comma] {Data/gonemaxfreq-n400-K600-r8.txt};
  \end{axis}
\end{tikzpicture}
\caption{\textmd{The analysis of frequency evolution for the \rcga on both \ronemax and \gonemax; with the left side displaying results for \ronemax and the right side for \gonemax; $n=400$ and $K=600$; $r\in\{0,\dots,7\}$, for a particular position.}}
\label{figure:freqplot}
 \end{figure}

\section{Conclusion}
\label{section:conclusion}

This work provides the first runtime analysis of the multi-valued \rcga on the generalized \gonemax function, deriving a high probability bound on its runtime. Furthermore, we have refined the runtime analysis of the \rcga on the \ronemax function, improving previous estimates and eliminating a logarithmic factor. In addition, for the first time, we include the analysis of frequency borders. These results contribute to a deeper understanding of EDAs' performance on multi-valued functions and open avenues for further research, particularly in analyzing the \rcga on more complex problems and exploring potential runtime improvements. Additionally, based on our experiments, we believe that there is potential to further improve the runtime bounds on \gonemax.

\paragraph{Acknowledgements}
This work has been supported by the Danish Council for Independent Research through grant 10.46540/2032-00101B.

%
\bibliographystyle{plain}
\bibliography{References}

\end{document}